%% file: ewrl_2026.tex
\documentclass{article}

\usepackage[final]{ewrl_2026}

\usepackage[utf8]{inputenc} 
\usepackage[T1]{fontenc}    
\usepackage{hyperref}       
\usepackage{url}            
\usepackage{booktabs}       
\usepackage{amsfonts}       
\usepackage{nicefrac}       
\usepackage{microtype}      
\usepackage{xcolor}         
\usepackage{amsmath} 

\usepackage{graphicx}
\usepackage{subcaption}
\usepackage{caption}
\usepackage{makecell}

\usepackage{enumitem}
\usepackage{comment}
\usepackage{amssymb,amsthm}
\usepackage{wrapfig}
\usepackage{mathtools}
\newtheorem{proposition}{Proposition}

\usepackage{algorithm,algorithmic}
\usepackage[table]{xcolor}

\title{Survival Reinforcement Learning: \\ 
Toward Scalable Self-Supervised RL}

\author{%
  Franki Nguimatsia Tiofack$^*$
  \\
  Inria and École Normale Supérieure\\
  PSL Research University\\
  Paris, France \\
   \\
  \And
  Fabian Schramm \\
  Inria and École Normale Supérieure\\
  PSL Research University\\
  Paris, France \\
  \And
  Théotime Le Hellard \\
  Inria and École Normale Supérieure\\
  PSL Research University\\
  Paris, France \\
  \And
  Justin Carpentier \\
  Inria and École Normale Supérieure\\
  PSL Research University\\
  Paris, France \\
}

\begin{document}

\maketitle

\begin{abstract}
While self-supervised Contrastive Reinforcement Learning (CRL) has shown remarkable depth-scaling capabilities, successfully using networks over 64 layers, scaled CRL still struggles with long-horizon goal-conditioned planning due to the uniformity-tolerance dilemma inherent in contrastive losses. 
We introduce Survival Reinforcement Learning (SRL), an online classification-based alternative that extends the survival value learning framework by maximizing the agent's dwell time at target goals.
SRL bypasses the structural constraints of CRL and mitigates the "bang-bang" control solutions inherent to survival frameworks, which often induce undesirable behavior in complex dynamical systems.
Evaluated across diverse robotic benchmarks, scaled SRL matches state-of-the-art CRL on manipulation tasks and outperforms it by 2x to 8x on stable, long-horizon locomotion tasks.
Our results provide strong additional evidence that classification-based methods may serve as a key primitive in the broader effort to scale reinforcement learning and an open-source implementation is available \href{https://github.com/Simple-Robotics/survival-reinforcement-learning}{online}.
\end{abstract}

\section{Introduction}
\label{sec:intro}

\input{sections/1-introduction}

\section{Related work}
\label{sec:related}

\input{sections/2-related-works}

\section{Background} 
\label{sec:background}

\input{sections/3-background-ewrl}

\section{Survival Reinforcement Learning}
\label{sec:method}

\input{sections/4-survival-reinforcement-learning}

\section{Experiments}
\label{sec:experiments}
\input{sections/5-experiments-ewrl}

\section{Discussion and Conclusion}
\label{sec:conclusion}

\input{sections/6-conclusion}

\section*{Acknowledgments}
This work has received support from the French government, managed by the National Research Agency, under the France 2030 program with the references Organic Robotics Program (PEPR O2R) and “PR[AI]RIE-PSAI” (ANR-23-IACL-0008).
This research was funded, in part, by l’Agence Nationale de la Recherche (ANR), PEPR O2R - AS2 (ANR-22-EXOD-0006).
The European Union also supported this work through the ARTIFACT project (GA no. 101165695) and the AGIMUS project (GA no. 101070165).
The Paris Île-de-France Région also supported this work in the frame of the DIM AI4IDF.
The authors gratefully acknowledge the support and resources provided by the CLEPS infrastructure at Inria Paris. This work was performed using HPC resources from the GENCI-IDRIS Jean-Zay cluster (Grant 2025-AD010616763,  Grant 2025-AD011017125 and Grant 2026-AD011014731R2).
Views and opinions expressed are those of the author(s) only and do not necessarily reflect those of the funding agencies.

\clearpage
\bibliographystyle{plainnat}
\bibliography{refs}

\appendix

\section{Appendix}
\label{sec:appendix}

\input{sections/7-appendix-ewrl}

\end{document}

%% file: sections/1-introduction.tex
In natural language processing and computer vision, the systematic scaling of neural network capacity and dataset size has driven unprecedented breakthroughs, establishing scale as a primary driver of generalist capabilities~\citep{Kaplan20ScalingLaws,dosovitskiy2021an}. Empirical scaling laws suggest that performance improves predictably as a power law of compute, data, and parameter count~\citep{henighan2020scaling}, with state-of-the-art models reaching trillions of parameters~\citep{brown2020language,fedus2022switch,hoffmann2022training,deepseek2026deepseek}. A growing body of work in RL seeks to replicate this success by designing architectures and algorithms whose performance scales reliably with network capacity~\citep{neumann2022scaling,schwarzer2023bigger,nauman2024bigger,lee2024simba,wang2025}. By connecting the conventional wisdom on self-supervised learning to the self-supervised RL setting, \citet{wang2025} provide the first compelling evidence of a genuine scaling phenomenon in self-supervised RL, demonstrating that increasing network depth in the classification-based approach Contrastive RL (CRL) yields substantial performance gains on challenging goal-reaching tasks.
Yet, this scaling phenomenon remains specific to CRL as all non-contrastive goal-conditioned baselines fail to consistently benefit from increased depth. Even the contrastive approach, scaled to 64 layers, remains a bottleneck for complex, long-horizon planning. This raises a natural question: \emph{can we design alternative self-supervised RL objectives that scale as favorably with network depth as CRL or unlock even stronger performance gains?}

The limited performance of scaled CRL on long-horizon tasks can be attributed in part to a well-known limitation. The uniformity-tolerance dilemma~\citep{wang2021understanding, huang2023model} arises because InfoNCE-style losses encourage representations to be uniformly spread across the embedding space, penalizing semantically similar states and destroying the geometric structure necessary for long-horizon planning. A promising alternative is survival learning-based RL~\citep{tiofack2026svl}, which frames policy evaluation as a first-hitting-time estimation problem. Rather than relying on contrastive objectives, it trains the critic by classifying whether a given state-action pair belongs to trajectories leading toward a goal, and at which time step along the trajectory the goal is reached. Its objective is grounded in maximum-likelihood estimation over observed and right-censored trajectories, a formulation that avoids the pitfalls of representation collapse associated with InfoNCE-style losses and positions survival learning-based RL as a compelling candidate for a scalable, online self-supervised RL algorithm.

In this paper, we introduce Survival Reinforcement Learning (SRL), an online self-supervised RL algorithm that extends the survival learning framework~\citep{tiofack2026svl} to account for post-goal reachability. In the standard formulation, the agent minimizes a functional of the minimum time to reach the goal from a given state, a criterion that induces bang-bang control solutions~\citep{evans1983introduction, maurer2004second}. Indeed, as in prior work~\citep{wang2025}, a more desirable policy may be one that not only drives the agent to reach the goal, but also stabilizes it there. To address this, we extend the survival learning framework with a dwell time at goal formulation that explicitly incentivizes the agent to stabilize at the goal state upon reaching it. Leveraging the network architecture proposed by prior work and GPU-accelerated simulation~\citep{bortkiewicz2024accelerating,wang2025}, we conduct an empirical study of SRL's scaling behavior with network depth, compare it against state-of-the-art scaled CRL, and ablate the contribution of the dwell time at goal formulation. As reported Fig.~\ref{fig:teaser}, scaled SRL achieves competitive results on challenging goal-reaching tasks, with particularly strong performance in the AntMaze environments, providing additional evidence that classification-based objectives could be a good choice for unlocking the scaling potential of RL.

\begin{figure}[t]
    \centering
    \includegraphics[width=\textwidth]{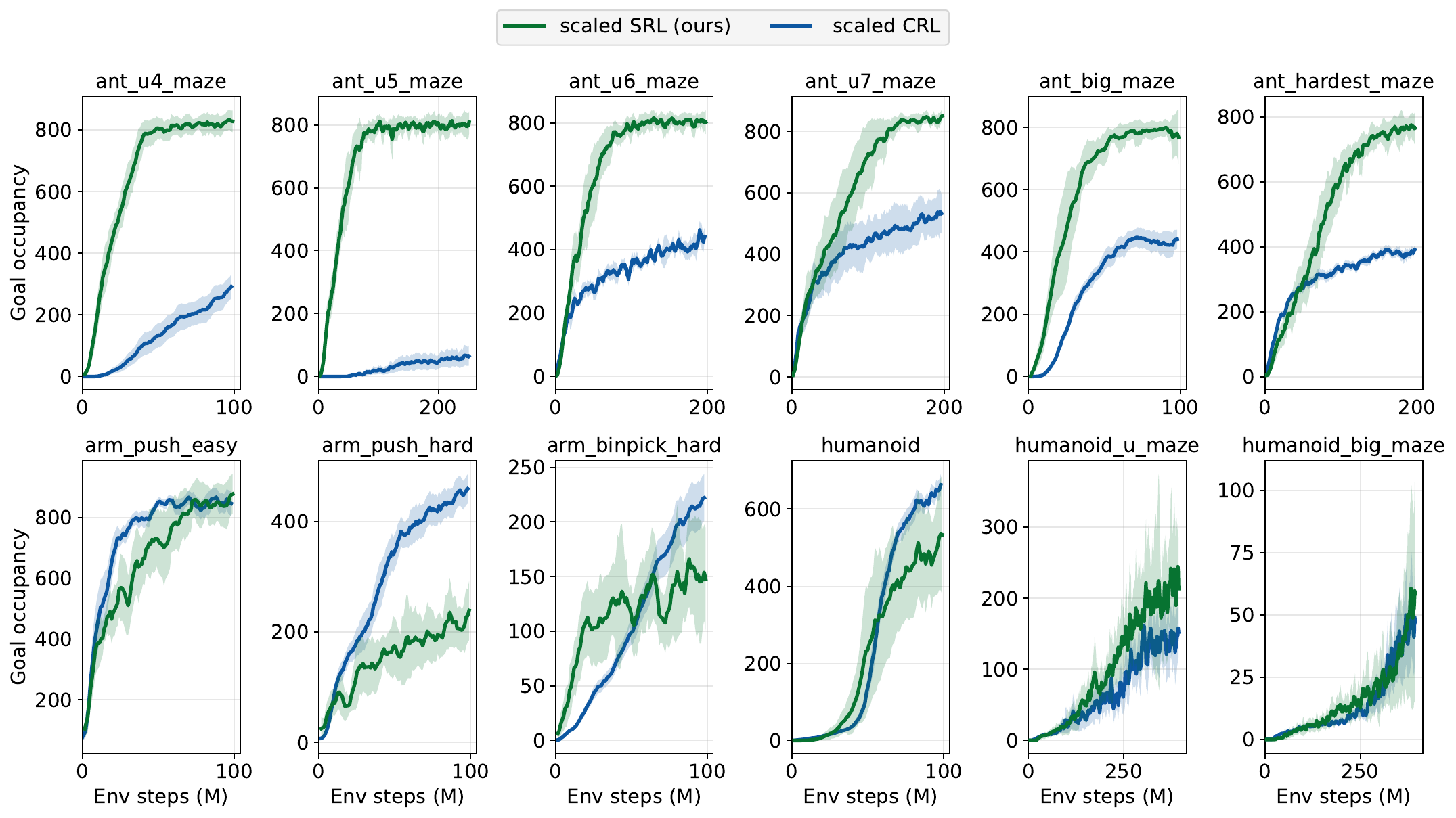}
    \caption{\textbf{Performance comparisons between SRL and CRL.} Time at goal of \textbf{SRL} (depth 32) against the state-of-the-art \textbf{CRL} (scaled up to depth 64,~\citet{wang2025}) on JaxGCRL.}
    \label{fig:teaser}
    \vspace{-2.1em}
\end{figure}

%% file: sections/2-related-works.tex
Designing architectures and algorithms that establish scaling laws for performance as a function of model and dataset size has been a long-standing challenge in the RL community~\citep{neumann2022scaling,ota2024framework,nauman2024bigger,lee2024simba,tuyls2023scaling}. Layer normalization~\citep{nauman2024bigger,ball2023efficient,nauman2024overestimation,lee2024simba} and residual connections~\citep{lee2024simba,nauman2024bigger,wang2025,ota2024framework} have emerged as the principal building blocks for stabilizing the training of wide and deep networks in RL. \citet{wang2025} further highlights the role of activation functions, empirically showing that Swish~\citep{ramachandran2018searching} is preferable to the widely adopted ReLU~\citep{nair2010rectified} when scaling contrastive RL. We adopt the architecture of~\citet{wang2025} as our backbone to evaluate the depth-scaling behavior of SRL. Additionally, to encode the temporal structure inherent to the survival-based formulation, we follow~\citet{tiofack2026svl} and employ the temporal basis described in Fig.~\ref{fig:network-architecture}.

On the algorithmic side, classification-based approaches have yielded promising results toward scalable RL. Based on the conjecture that classification objectives are inherently more stable than their regression-based counterparts~\citep{torgo1996regression}, and thus better suited to deep networks, \citet{farebrother2024stop} reformulate value-based RL as a classification problem by discretizing the temporal-difference target into a categorical cross-entropy loss. \citep{wang2025} advances a related argument, observing that the InfoNCE objective underlying CRL is a generalization of cross-entropy loss: rather than regressing value estimates, it learns by classifying whether state-action pairs belong to trajectories leading toward a given goal. 
Our work follows this line of reasoning. 
The survival learning objective frames policy evaluation as classification, predicting not only whether state-action pairs belong to trajectories leading toward a given goal, but also the time instant along the trajectory at which the goal is reached.

Several previous works formulate self-supervised RL as a classification problem~\citep{ghosh2019learning,eysenbach2020c,eysenbach2022contrastive,yang2022rethinking,tiofack2026svl,mezghani2023learning}, but diverge in how they treat post-goal behavior. Contrastive methods~\citep{eysenbach2020c,eysenbach2022contrastive} optimize goal-occupancy, training the agent to reach and stabilize at the goal. A complementary line of work~\citep{ghosh2019learning,yang2022rethinking,mezghani2023learning,tiofack2026svl} focuses on goal-reaching, leaving post-goal behavior unmodeled, an omission that can produce undesirable outcomes for complex or unstable problems such as humanoid locomotion. This pathology is acute in the survival learning formulation~\citep{tiofack2026svl}, where the policy minimizes a functional of the first-hitting time, whose optimal solution leads to bang-bang control~\citep{evans1983introduction,maurer2004second} that drives the agent toward the goal as quickly as possible, with no incentive to decelerate, stabilize, or remain upon arrival. The dwell time at goal formulation introduced in this work mitigates this limitation by extending the survival learning objective to explicitly incentivize the agent to stabilize at the goal.

%% file: sections/3-background-ewrl.tex
\textbf{Goal-conditioned RL.} We consider a goal-conditioned Markov decision process (MDP) \mbox{$\mathcal{M} = (\mathcal{S}, \mathcal{A}, \mathcal{G}, \mathcal{P}, r_g, \varphi, p_g, \mu, \gamma)$}, where $\mathcal{S}$, $\mathcal{A}$, and $\mathcal{G}$ denote the state, action, and goal spaces, respectively. Let $\Delta(\mathcal{X})$ denote the set of probability distributions over a set $\mathcal{X}$. The transition dynamics $\mathcal{P} : \mathcal{S} \times \mathcal{A} \to \Delta(\mathcal{S})$ governs the environment; $\varphi : \mathcal{S} \to \mathcal{G}$ projects states into the goal space; $p_g \in \Delta(\mathcal{G})$ is the goal distribution; $\mu : \mathcal{G} \to \Delta(\mathcal{S})$ is the goal-conditioned initial-state distribution; and $\gamma \in [0,1)$ is the discount factor. At the start of each episode, a goal $g \sim p_g$ is sampled, and the agent is initialized at $s_0 \sim \mu(\cdot \mid g)$. The agent then follows a goal-conditioned policy $\pi : \mathcal{S} \times \mathcal{G} \to \Delta(\mathcal{A})$. 
Following prior work~\citep{andrychowicz2017hindsight,chane2021goal,blier2021learning}, we adopt a sparse reward formulation by setting $r_g(s_t, a_t) = -\mathbf{1}_{\{\varphi(s_{t+1}) \notin \mathcal{B}_\epsilon(g)\}}$ with \mbox{$\mathcal{B}_\epsilon(g) = \left\{s \in \mathcal{S} \mid \|\varphi(s) - g\| \leq \epsilon\right\} $}. 
For a given goal-conditioned policy $\pi$, the corresponding state-action value function is  defined as:
\begin{equation}
\label{Q_function}
Q_g^\pi(s,a) =  \mathbb{E}_{\tau \sim \rho^\pi(\cdot |g)} \left[\sum_{t=0}^\infty\gamma^t r_g(s_t,a_t) |\substack{
    s_0=s \\ a_0 = a
} \right],
\end{equation}
where $\rho^\pi(\tau|g) = \mu(s_0 \mid g)\prod_{t=0}^\infty \pi(a_t \mid s_t, g) \mathcal{P}(s_{t+1} \mid s_t, a_t)$ denotes the likelihood of the rollout trajectory $\tau = (s_0, a_0, s_1, a_1, \dots)$ under $\pi$. 

\textbf{Metrics.}  We evaluate a goal-conditioned policy $\pi$ over fixed horizon $H$ using two metrics:
\begin{equation}
\label{eqn:sr-go}
\mathrm{SR}(\pi) = \mathbb{E}_{\substack{g \sim p_g \\ \tau \sim \rho^\pi(\cdot\mid g)}} \!\left[ \mathbf{1}_{\{ \exists\, t \le H : \varphi(s_t) \in \mathcal{B}_\epsilon(g) \}} \right] \,\text{ and  }\,  \mathrm{GO}(\pi) = \mathbb{E}_{\substack{g \sim p_g \\ \tau \sim \rho^\pi(\cdot\mid g)}} \!\left[ \sum_{t=1}^H \mathbf{1}_{\{ \varphi(s_t) \in \mathcal{B}_\epsilon(g) \}} \right].
\end{equation}
The success rate $\mathrm{SR}(\pi)$ measures the probability that the agent reaches the goal region at least once during an episode, capturing task completion regardless of subsequent behavior. Goal-occupancy $\mathrm{GO}(\pi)$ measures the expected number of timesteps spent within the goal region, capturing the agent's ability to reach \emph{and sustain} goal satisfaction. The two metrics can diverge significantly: an agent may achieve a high success rate by transiently passing through the goal region while accumulating negligible goal occupancy due to overshooting or instability upon arrival.
Depending on the algorithmic design, the learned policy may explicitly target goal-occupancy or success rate.

\textbf{Policy evaluation via survival analysis.} In the GCRL setting, \citet{tiofack2026svl} reformulate policy evaluation through survival analysis framework~\citep{kaplan1958nonparametric,katzman2018deepsurv}. 
They model the first-hitting time of the goal region, a random variable defined as the minimum number of steps required to transition from state $s$ to goal $g$ when taking action $a$ and subsequently following policy $\pi$: 
\begin{equation} 
T_g^\pi(s,a) = \inf \Bigl\{t \geq 0 : \varphi(s_{t+1}) \in \mathcal{B}_\epsilon(g) \mid \substack{s_0 = s,\ a_0 = a \\ a_t \sim \pi(\cdot \mid s_t, g),\ s_{t+1} \sim \mathcal{P}(\cdot \mid s_t, a_t)}\Bigr\},
\label{first_hit}
\end{equation}
with the convention that $T_g^{\pi}(s,a) = \infty$ if the goal is never reached. Under the assumption that each episode terminates upon reaching the goal region --~equivalently, that every goal state is absorbing~-- the goal-conditioned action-value function admits the following closed-form relationship with $T_g^\pi(s,a)$: 
\begin{equation} 
Q_g^\pi(s,a) = -\sum_{t=0}^{\infty}\gamma^t \Pr\bigl(T_g^\pi(s, a) > t\bigr). 
\label{svl_q_function} 
\end{equation}
This identity, obtained by adapting Proposition~4.1 of \citet{tiofack2026svl} to the action-value function, shows that critic learning reduces to estimating the survival probability \mbox{$S^\pi(t\mid s,a, g) \triangleq \Pr(T_g^\pi(s, a) > t)$}. This is estimated via a parametrized hazard model 
\begin{equation} 
h^{\pi}_\theta(t \mid s,a,g) \triangleq \Pr\bigl(T_g^\pi(s, a) = t \mid T_g^\pi(s, a) \geq t\bigr), 
\end{equation} 
which gives the probability that the goal is reached at step $t$, given that it has not been reached before. 
The survival probability is then recovered through the product formula \mbox{$S^\pi(t\mid s,a,g)=\prod_{k=0}^{t}\bigl(1-h^\pi_\theta(k\mid s,a,g)\bigr)$}, and the hazard model is trained by maximizing the right-censored likelihood over collected trajectories:
\begin{equation}
\label{nll}
\begin{aligned}
\mathcal{L}(\theta) &= -\mathbb{E}_{(s,a,g,t_h,c,\delta) \sim \mathcal{D}}\!\left[
\delta\,\ell_{t_h}(\theta)+(1-\delta)\,\ell_c(\theta)\right]\\
 \text{ with } \ell_c(\theta) &=
 \sum_{j=0}^{c}\log\bigl(1-h_{\theta}^\pi(j\mid s,a,g)\bigr)\\ 
 \text{and }  \ell_{t_h}(\theta) &= \log h_{\theta}^\pi(t_h\mid s,a,g)
+\sum_{j=0}^{t_h-1}\log\bigl(1-h_{\theta}^\pi(j\mid s,a,g)\bigr), 
\end{aligned}
\end{equation}
where $\delta\in\{0,1\}$ indicates whether the goal is reached within $c$ steps, $t_h$ denotes the hitting time when $\delta=1$. For a complete derivation, see~\citet {tiofack2026svl}. Note that this formulation does not explicitly model post-goal achievement and primarily focuses on success rate.

%% file: sections/4-survival-reinforcement-learning.tex
We introduce Survival Reinforcement Learning (SRL) by first extending the survival learning framework with a dwell-time-at-goal formulation that accounts for post-goal-reaching behavior, and then building an online goal-conditioned actor-critic algorithm on top of this extended critic objective.

\subsection{Dwell-time-at-goal formulation}

By modeling only the first-hitting time, the original survival learning formulation is indifferent to post-achievement behavior. As is well understood in optimal control theory, minimizing a first-hitting-time functional typically yields bang-bang solutions~\citep{evans1983introduction,maurer2004second,bellman1956bang}. The agent reaches the goal as rapidly as possible, with no incentive to decelerate, stabilize, or remain. In practice, as illustrated in Section~\ref{sec:experiments}, this manifests as a humanoid that lunges toward the target at high velocity and falls upon arrival, or as a manipulator that overshoots the target configuration. 
Post-goal stability is therefore necessary for reliable performance on locomotion and manipulation tasks. We address this by reformulating the hitting-time criterion to require the agent to match a short sequence of consecutive goals.

\textbf{Extended hitting time.} Given a state trajectory $(s_0, s_1, \ldots)$, we define the \emph{$k$-length goal sequence} sampled at offset $u \geq 0$ as:
\begin{equation}
\label{eq:goal_seq}
\mathbf{g}^{(k)} = \bigl(\varphi(s_{u}),\varphi(s_{u+1}),\ldots,\varphi(s_{u+k-1})\bigr) \in \mathcal{G}^{k}.
\end{equation}
The first-hitting time is then extended to require $k$ consecutive goal matches:
\begin{equation}
\label{action_first_hit_extend}
T_{\mathbf{g}^{(k)}}^\pi(s,a) = \inf \left\{t \geq 0 :\varphi(s_{t+i+1}) \in \mathcal{B}_\epsilon(g_i), \  \forall i\in\{0,\ldots,k-1\} \middle| \substack{s_0=s,\ a_0=a \\ a_t\sim\pi(\cdot\mid s_t,g_0) \\ \ s_{t+1}\sim \mathcal{P}(\cdot\mid s_t,a_t)}
\right\},
\end{equation}
where $g_i$ denotes the $i$-th component of $\mathbf{g}^{(k)}$. The original formulation of Eq.~\ref{first_hit} is recovered as the special case $k=1$. When $g_i=g$ for all $i\in\{0,\ldots,k-1\}$, the event requires the agent to remain
within $\mathcal{B}_\epsilon(g)$ for $k$ consecutive timesteps. An agent that enters the goal region and immediately exits does not trigger the survival event and is penalized accordingly.

\textbf{Extended value function.} This formulation inherits the full survival analysis structure of~\citet{tiofack2026svl}. Assuming episode termination upon $\mathbf{g}^{(k)}$ achievement, we define the extended action-value function evaluating the ability of policy $\pi$ to achieve $\mathbf{g}^{(k)}$ as:
\begin{equation}
\label{eq:Q_srlpp}
Q^\pi_{\mathbf{g}^{(k)}}(s,a) = -\sum_{t=0}^{\infty}
\gamma^t\Pr\bigl(T_{\mathbf{g}^{(k)}}^\pi(s,a) > t\bigr),
\end{equation}
which extends the identity in Eq.~\ref{svl_q_function} to the dwell time at goal setting. Following~\citep{tiofack2026svl}, we estimate $Q^\pi_{\mathbf{g}^{(k)}}(s,a)$ by training a parametrized hazard model
\begin{equation}
h^{\pi}_\theta(t \mid s,a,\mathbf{g}^{(k)}) \triangleq \Pr\bigl(T_{\mathbf{g}^{(k)}}^\pi(s,a)=t \mid T_{\mathbf{g}^{(k)}}^\pi(s,a) \geq t\bigr)
\end{equation}
via the right-censored negative log-likelihood objective of:
\begin{equation}
\label{nll}
\mathcal{L}_{\text{hazard}}(\theta) = -\mathbb{E}_{(s,a,\mathbf{g}^{(k)},t_h,c,\delta) \sim \mathcal{D}}\left[
\delta\,\ell_{t_h}(\theta)+(1-\delta)\,\ell_c(\theta)\right],
\end{equation}
where the event log-likelihood $\ell_{t_h}$ and censored log-likelihood $\ell_c$ are given by:
\begin{equation}
\ell_{t_h}(\theta)\! =\! \log h_{\theta}^\pi({t_h}\mid s,a,\mathbf{g}^{(k)})
+\sum_{j=0}^{{t_h}-1}\log\bigl(1-h_{\theta}^\pi(j\mid s,a,\mathbf{g}^{(k)})\bigr), \ \
\ell_c(\theta)\!=\!\sum_{j=0}^{c}\log\bigl(1-h_{\theta}^\pi(j\mid s,a,\mathbf{g}^{(k)})\bigr)
\end{equation}
The survival probability $\Pr(T_{\mathbf{g}^{(k)}}^\pi(s,a)>t) = \prod_{j=0}^{t}\bigl(1-h^\pi_{\theta}(j\mid s,a,\mathbf{g}^{(k)})\bigr)$ is then substituted into Eq.~\ref{eq:Q_srlpp} to recover the extended value function. 
\subsection{Policy update}
\begin{wrapfigure}{r}{0.51\textwidth}
\vspace{-3.5em}
\vspace{-\baselineskip}
\begin{minipage}{0.49\textwidth}
\begin{algorithm}[H]
\footnotesize
\caption{\textbf{Survival Reinforcement Learning}}
\label{alg:srl}
\begin{algorithmic}[1]
\STATE Initialize critic $\theta$, actor $\psi$, replay buffer $\mathcal{D}$
\WHILE{not converged}
    \WHILE{Data collection}
        \STATE Sample goal $g \sim p_g$
        \STATE Policy rollout $\tau \sim \rho^{\pi_\psi}(\cdot \mid g)$
        \STATE Construct $(s,\! a,\! \mathbf{g}^{(k)},\! t_h,\! c, \! \delta)$, store in $\mathcal{D}$
    \ENDWHILE
    \FOR{each gradient step}
        \STATE Sample $(s, a, \mathbf{g}^{(k)}, t_h, c, \delta)$ from $\mathcal{D}$
        \STATE $\theta \leftarrow \theta - \lambda_\theta \nabla_\theta \mathcal{L}_{\text{hazard}}(\theta)$
        \STATE $\psi \leftarrow \psi - \lambda_\psi \nabla_\psi \mathcal{L}_{\text{actor}}(\psi)$
    \ENDFOR
\ENDWHILE
\end{algorithmic}
\end{algorithm}
\end{minipage}
\vspace{-2em}
\end{wrapfigure}

This extension requires no new loss function or architectural modification beyond the original survival learning framework.
The actor $\pi_\psi(\cdot\mid s,g)$ is conditioned on a single point goal $g\in\mathcal{G}$, preserving the standard goal-conditioned RL interface. During actor updates, we instantiate the goal sequence as $\mathbf{g}^{(k)}=(g,g,\ldots,g)$, which aligns the critic signal with the dwell time at goal objective and incentivizes the agent to reach and remain at $g$. The stochastic Gaussian policy is trained by minimizing the maximum-entropy objective:
\begin{equation}
\label{eq:actor_srlpp}
\mathcal{L}_{\text{actor}}(\psi) = \mathbb{E}_{\substack{(s,g)\sim\mathcal{D}, \
a\sim\pi_\psi(\cdot\mid s,g)}}\Bigl[
-Q^\pi_{(g,\ldots,g)}(s,a) + \alpha\log\pi_\psi(a\mid s,g)\Bigr],
\end{equation}
where $\alpha > 0$ is the entropy temperature~\citep{haarnoja2018soft}. Although SRL is formally on-policy since the hazard is estimated under the policy that collected the data, we follow standard practice~\citep{haarnoja2018soft, fujimoto2018addressing, eysenbach2022contrastive} and perform multiple gradient steps per transition, treating the replay buffer as an off-policy dataset. 
The pseudocode is provided in Algorithm ~\ref{alg:srl}, and we refer to the Appendix for additional details.

\subsection{Goal-occupancy bound} 
The following proposition justifies our dwell-time formulation: goal-occupancy is lower-bounded by the remaining horizon upon arrival, discounted by the probability of premature exit.
Sec.~\ref{sec:exp-stacking} illustrates that, in practice, minimizing the hitting time $T^\pi_{\mathbf{g}^{(k)}}$ decreases the exit probability $p^\pi_k$.

\begin{proposition}[Goal-occupancy lower bound]
\label{prop:occupancy-bound}
For a given policy $\pi$ and fixed episode length $H$, let
\begin{equation*}
p_k^\pi = \max_{t \in \{1,\ldots,H\}} \mathbb{E}_{\substack{g \sim p_g,\, s_0 \sim \mu(\cdot|g), \,  a_0 \sim \pi(\cdot \mid s_0, g) \\  \ \tau \sim \rho^\pi(\cdot|g,s_0, a_0)}}\!\left[\mathbf{1}_{\{\varphi(s_t) \notin \mathcal{B}_\epsilon(g)\}} \mid T_{\mathbf{g}^{(k)}}^\pi(s_0, a_0) < t\right],
\end{equation*}
where $\mathbf{g}^{(k)}=(g,g,\ldots,g)$ and  $T_{\mathbf{g}^{(k)}}^\pi(s_0, a_0)$ is the extended hitting time defined in Eq.~\ref{action_first_hit_extend}. Then:
\begin{equation}
\mathbb{E}_{\substack{g \sim p_g \\ \tau \sim \rho^\pi(\cdot\mid g)}} \left[ \sum_{t=1}^H \mathbf{1}_{\{ \varphi(s_t) \in \mathcal{B}_\epsilon(g) \}} \right] \geq (1 - p_k^\pi)\left(H - \mathbb{E}_{\substack{g \sim p_g, \ s_0 \sim \mu(\cdot|g) \\ \ a_0 \sim \pi(\cdot \mid s_0, g)}}\left[T_{\mathbf{g}^{(k)}}^\pi(s_0, a_0) \wedge H\right]\right).
\label{inequality}
\end{equation}
\end{proposition}

\textbf{Proof.} We refer readers to Appendix~\ref{app:stacking-proof}.

%% file: sections/5-experiments-ewrl.tex
\begin{figure*}[t]
    \centering
    \includegraphics[width=\textwidth]{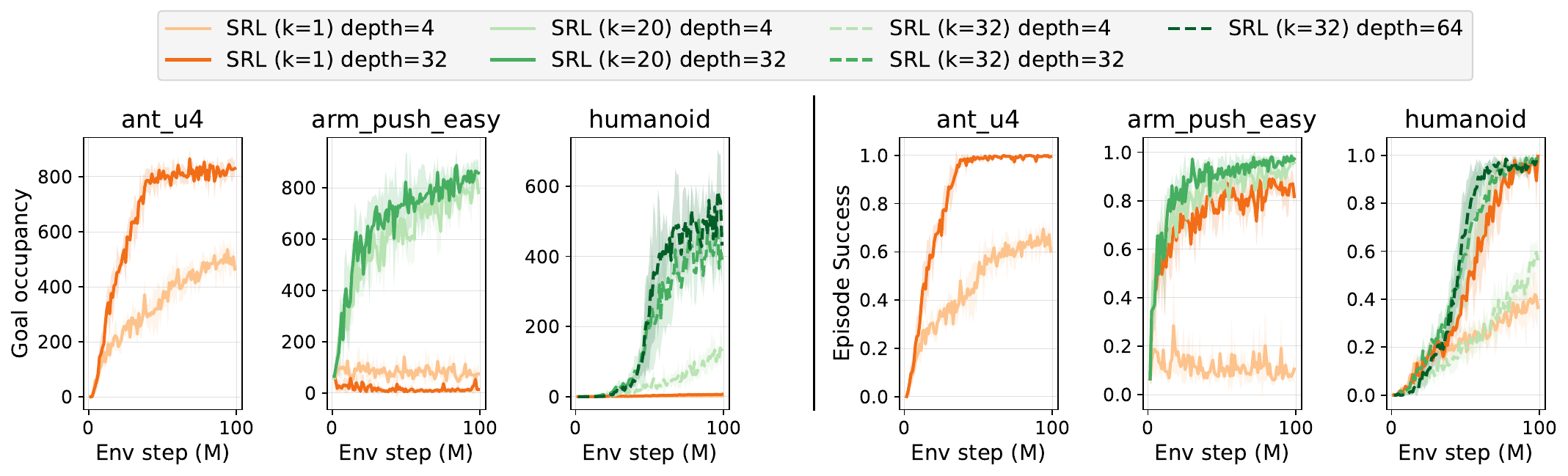}
    \caption{\small \textbf{SRL on two metrics} for varying network depth and goal sequence $k$. The dwell time at goal ($k>1$) is essential for \texttt{Arm Push Easy} and \texttt{Humanoid} to exhibit scaling behavior in goal occupancy.}
    \label{dwell_plot}
    \vspace{-1em}
\end{figure*}

We structure our evaluation around three main questions: 
(i) How important is the dwell time at goal formulation?; 
(ii) How does SRL scale with network depth?;
(iii)~How tight is the occupancy lower bound of Proposition~\ref{prop:occupancy-bound}? 
We use CRL~\citep{eysenbach2022contrastive} as our baseline, as it was shown to be the principal goal-conditioned RL algorithm to consistently scale with network depth~\citep{wang2025}. SRL implementation will be released upon paper acceptance. 
\subsection{Experimental Setup}

We implement SRL based on~\citet{wang2025}, which adapts JaxGCRL~\citep{bortkiewicz2024accelerating} to enable fast online GCRL experiments via Brax~\citep{freeman2021brax}. 
We evaluate on a diverse set of locomotion, navigation, and robotic manipulation tasks (see Fig.~\ref{fig:jaxGCRL-envs}) using the two metrics defined in Eq.~\ref{eqn:sr-go} and fix the episode length to $H=1000$.

Following prior work~\citep{wang2025,lee2024simba,nauman2024bigger}, all networks use residual blocks~\citep{he2016deep}  with four Dense layers each, Layer Normalization~\citep{ba2016layer}, Swish activation~\citep{ramachandran2018searching}, and temporal basis~\citep{tiofack2026svl}, see Appendix~\ref{app:network-architectures}. The depth is the total number of dense layers across all residual blocks. The actor and both critic networks have the same depth and are scaled jointly. We use the piecewise-constant survival parameterization.

\subsection{Why the dwell-time-at-goal formulation?}
\label{sec:dwell-goal-formulation}

\vspace{-0.2em}
\textbf{Stable locomotion (\texttt{Ant U4 Maze}).} 
On this task, the standard survival learning formulation ($k=1$) scales well with network depth and saturates on both metrics, as shown in the first and fourth columns of Fig.~\ref{dwell_plot}. With sufficient capacity (32 layers), the policy learns to move the ant as quickly as possible toward the goal, yielding a bang-bang control solution, as studied by optimal control theory~\citep{maurer2004second, evans1983introduction}. The higher arrival speed causes a slight overshoot, followed by a correction, resulting in a $100\%$ success rate and a goal occupancy above $800/1000$. This behavior is shown in the first row of Fig.~\ref{fig:behavior-imgs}.

\textbf{Challenging locomotion (\texttt{Humanoid}).}  
When $k=1$, the success rate scales well with depth, rising from below $40\%$ at depth 4 to $100\%$ at depth 32. Goal-occupancy remains negligible at all depths, as shown in the third and sixth columns of Fig.~\ref{dwell_plot}. Again, the policy converges to a bang-bang solution: the \texttt{Humanoid} lunges toward the target at full speed with no incentive to decelerate or stabilize, and falls upon arrival. This behavior is illustrated in the second row of Fig.~\ref{fig:behavior-imgs}. The dwell time at goal formulation corrects this by incentivizing the agent to decelerate when approaching the goal, enabling it to stabilize upon arrival, as shown in the third row of Fig.~\ref{fig:behavior-imgs}. With $k=32$, both metrics scale consistently with depth, as depicted in the third and sixth columns of Fig.~\ref{dwell_plot}.

\textbf{Manipulation (\texttt{Arm Push Easy}).} A similar pattern emerges for manipulation, as shown in the second and fifth columns of Fig.~\ref{dwell_plot}. With $k=1$, scaling the depth from 4 to 32 layers improves the success rate from roughly $40\%$ to $90\%$, but goal occupancy remains near zero throughout. The standard objective incentivizes the arm to strike the cube with excessive force, satisfying the transient success condition while failing to achieve sustained occupancy. Setting $k=20$ resolves this: the dwell time at goal formulation seamlessly saturates both metrics.

\begin{figure}[t]
    \centering
    \includegraphics[width=\linewidth]{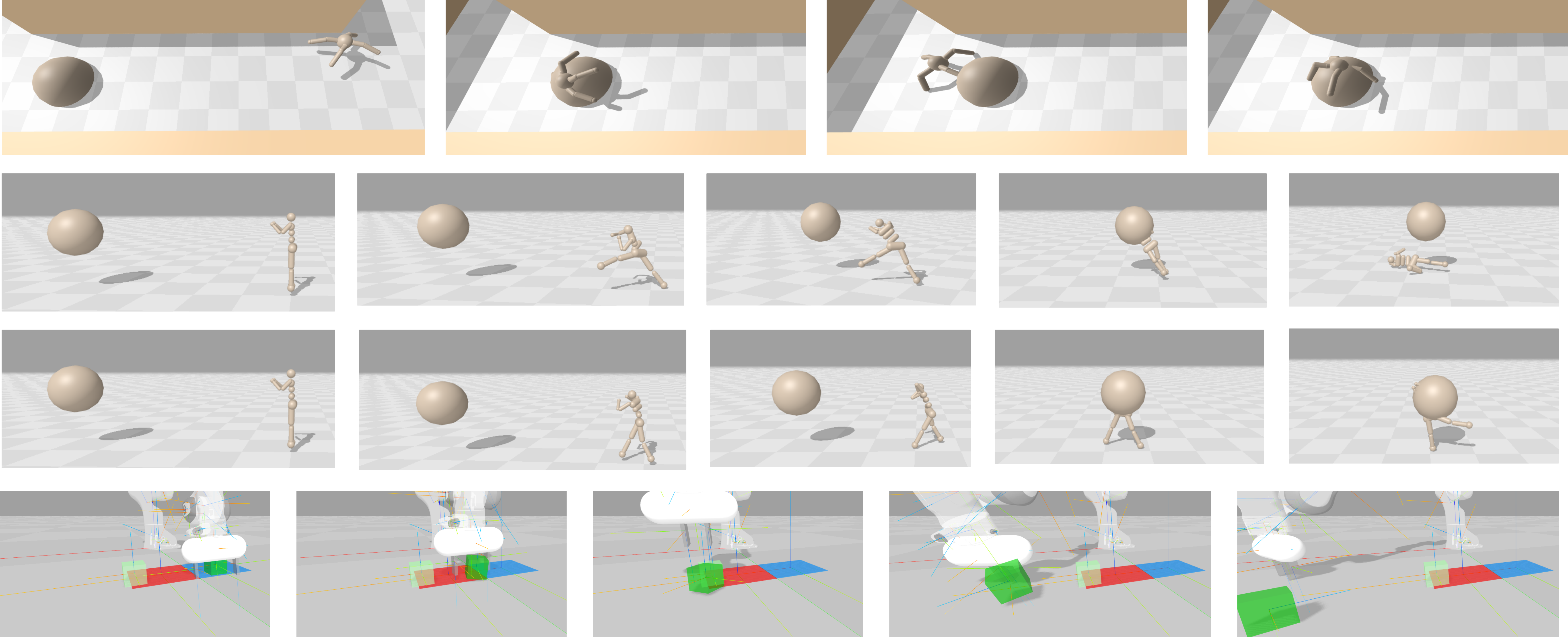}
\caption{\small \textbf{Behavioral visualization of SRL.} \textit{First row:} with \mbox{$k=1$}, the \texttt{Ant} overshoots the target and must reverse direction to return. \textit{Second row:} with \mbox{$k=1$}, the \texttt{Humanoid} lunges toward the target at full speed and collapses upon arrival. \textit{Third row:} with the dwell-goal mechanism (\mbox{$k=32$}), the \texttt{Humanoid} decelerates as it approaches the target to maintain dynamic balance, achieving sustained goal occupancy. \textit{Fourth row:} The arm strikes the cube with excessive force, causing it to overshoot the target position.}
    \label{fig:behavior-imgs}
\end{figure}

 \vspace{-0.2em}
\subsection{Scaling depth in SRL}
\label{sec:exp-main}
\vspace{-0.2em}
Fig.~\ref{fig:teaser} compares scaled CRL and scaled SRL across 12 environments, with SRL network depth scaled up to 32 layers. Despite using half the network depth of scaled CRL, SRL outperforms it by a substantial margin in 7 of 12 environments, achieving an average goal occupancy of $800/1000$ in most of those environments and remaining competitive in the rest. For all environments except \texttt{Ant U6 Maze} and \texttt{Ant U7 Maze}, the scaled CRL results were provided by~\citet{wang2025}.
Fig.~\ref{fig:depth-scaling-srl} shows the scaling performance of SRL across JaxGCRL for varying network depths.

\textbf{Ant navigation.} SRL excels across all AntMaze environments, independently of maze topology and planning horizon. At depth 32, SRL demonstrates significantly faster convergence and achieves goal-occupancy values that are two to four times higher than those of the contrastive baseline. SRL rapidly saturates at an average goal occupancy of $800/1000$ across all maze topologies from \texttt{Ant U4 Maze} to \texttt{Ant Hardest Maze}. Given a maximum episode horizon of $H = 1000$ steps, this represents near-optimal behavior: the agent requires only approximately $200$ steps to traverse any maze and reach the target, after which it successfully occupies the goal region for the remainder of the episode.

\textbf{Humanoid.} On the challenging \texttt{Humanoid} environment, SRL begins to succeed at depth 8 and achieves a goal-occupancy of nearly $600/1000$ when scaled up to depth 32. The resulting policies can stabilize at a given target position, as illustrated in the third row of Fig.~\ref{fig:behavior-imgs}. On \texttt{Humanoid U Maze} and \texttt{Humanoid Big Maze}, SRL exhibits scaling performance comparable to CRL~\citep{wang2025}.

\textbf{Manipulation.} On \texttt{Arm Push Easy}, SRL reaches near-optimal goal-occupancy ($800/1000$) at depth 4 and remains consistent across all depths, whereas CRL requires depth 16 to achieve comparable performance. On \texttt{Arm Binpick Hard}, a performance jump is observed from depth 4 to depth 8. The \texttt{Arm Push Hard} task remains a significant challenge for SRL, due to exploration bottlenecks: since we use vanilla SAC, task-agnostic entropy maximization struggles to discover the initial cube contact (a known failure mode on this environment, leading to trivial HER targets~\citep[Appendix A.5]{bortkiewicz2025accelerating}).

\subsection{SRL on extreme long-horizon tasks}

\begin{wrapfigure}{r}{0.55\textwidth}
    \vspace{-1.6em}
    \centering
    \includegraphics[width=\linewidth]{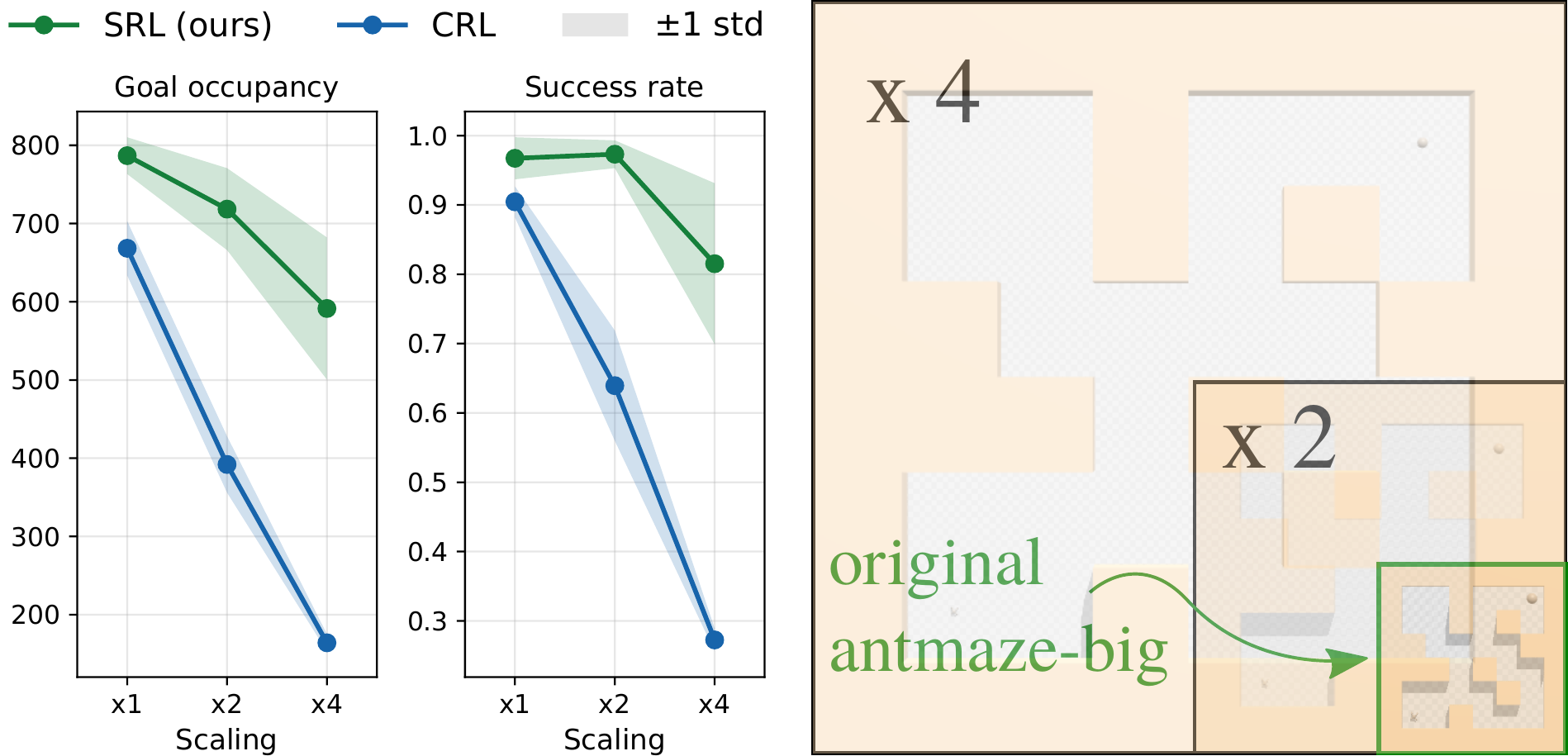}
       \caption{\small \textbf{SRL on extreme long-horizon tasks.} \textit{(left)} As the planning horizon increases beyond standard longest horizon, scaled SRL maintains consistent performance against CRL. \textit{(right)} Visual comparison of the three maze scales.}
        \label{fig:extreme-ant}
        \vspace{-0.5em}
\end{wrapfigure}

The results of Sec.~\ref{sec:exp-main} demonstrate that scaled SRL achieves near-optimal behavior on every \texttt{AntMaze} environment, from \texttt{U4} to \texttt{Hardest}. To probe the limits of SRL on long horizons beyond prior settings~\citep{wang2025}, we use the geometric flexibility of the maze suite. 
Concretely, we take the layout requiring the longest planning horizon, \texttt{Ant Big Maze}, and scale the maze's physical size by increasing \href{https://github.com/MichalBortkiewicz/JaxGCRL/blob/5b6d5d772c834a8269f492d9fe99df13339882aa/jaxgcrl/envs/ant_maze.py#L164}{\texttt{maze\_size\_scaling}} from its default value of $4.0$ to $8.0$ and $16.0$, effectively doubling and quadrupling the maze dimensions. We scale data collection to \textbf{one billion} transitions and retrain both scaled CRL and scaled SRL (depth=32). 
This modification has two compounding effects on task difficulty. First, the minimum number of steps required to traverse the maze grows linearly with the scaling factor, since path lengths increase proportionally, at \mbox{\texttt{maze\_size\_scaling} $= 16.0$}, the agent must navigate an environment four times larger than the standard benchmark. Second, the goal region $\mathcal{B}_\epsilon(g)$ remains fixed in absolute size, making it proportionally smaller and harder to reach precisely as the maze grows. Together, these effects place the task in the extreme long-horizon regime, exacerbating the curse of horizon~\citep{park2025horizon,myers2025horizon}.
As shown in Fig.~\ref{fig:extreme-ant}, the CRL baseline collapses under this extended horizon  ($<30\%$ success) when the maze is quadrupled. Conversely, SRL demonstrates remarkable robustness, maintaining over $80\%$ success and sustained goal occupancy even at the largest scale.

\begin{wrapfigure}{r}{0.55\textwidth}
    \vspace{-5.0em}
    \centering
    \includegraphics[width=\linewidth]{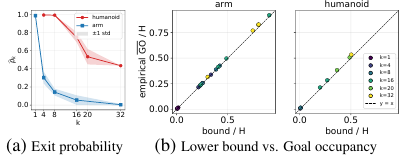}    
    \caption{\small \textbf{Ablation study.} \textbf{(a)} $\hat{p}_k$ decreases monotonically with $k$; \texttt{Humanoid} requires larger $k$ than \texttt{Arm Push Easy} to reduce exit probability, reflecting the difficulty of dynamic balancing versus quasi-static manipulation. \textbf{(b)} The lower bound from Proposition~\ref{prop:occupancy-bound} tracks empirical goal-occupancy closely along the $y=x$ diagonal (axes normalized by $H$), with larger $k$ shifting points toward higher occupancy.}
    \label{fig:stacking-validation}
    \vspace{-1.0em}
\end{wrapfigure}
\subsection{Ablation study}
\label{sec:exp-stacking}
We evaluate SRL on \texttt{Humanoid} and \texttt{Arm Push Easy} on $200$ rollouts post-training at network depth 16 and $k \in \{1, 4, 8, 16, 32\}$.

\textbf{Empirical exit probability.} Fig.~\ref{fig:stacking-validation}a reports the empirical exit probability $\hat{p}_k$. At $k = 1$, both environments yield $\hat{p}_k \approx 1.0$: a single in-goal step provides no incentive to stabilize. Increasing $k$ reduces $\hat{p}_k$ at different rates. On \texttt{Arm Push Easy}, $k = 4$ already drives $\hat{p}_k$ to $\approx 0.30$, decreasing monotonically to $\approx 0$ at $k = 32$. On \texttt{Humanoid}, $\hat{p}_k$ remains near $1.0$ for $k \leq 8$, drops to $\approx 0.75$ at $k = 16$, and reaches $\approx 0.43$ at $k = 32$. This contrast reflects task difficulty: the arm naturally holds the cube at rest once positioned, whereas the humanoid must continuously balance against gravity, requiring a larger $k$ to achieve sustained occupancy.

\textbf{Empirical occupancy lower bound.} Fig.~\ref{fig:stacking-validation}b validates Proposition~\ref{prop:occupancy-bound} by plotting the estimated lower bound against the empirical goal-occupancy, both normalized by $H$. Data points lie closely along the $y = x$ line across $k$, confirming that increasing $k$ improves the lower bound and drives higher occupancy. Runs with $k = 1$ cluster near the origin, while larger $k$ shifts points along the diagonal, reaching $\overline{\mathrm{GO}}/H \approx 0.9$ on \texttt{Arm Push Easy} and $\approx 0.5$ on \texttt{Humanoid} at $k \in \{20, 32\}$.

%% file: sections/6-conclusion.tex
We introduced SRL, an online self-supervised RL algorithm that extends the survival-learning framework of \citet{tiofack2026svl} with a dwell time at goal formulation.
Without requiring additional loss terms or architectural changes, this formulation aligns the optimization target with the goal occupancy metric and recovers stable post-arrival behavior. 
Empirically, scaled SRL completely solves every \texttt{AntMaze} topology at depth $32$, outperforming the depth-$64$ contrastive baseline of~\citet{wang2025}.
Our results extend prior findings of \cite{wang2025} that the choice of self-supervised objective governs depth scaling by showing that the survival objective scales at least as gracefully as the contrastive one and substantially more favorably on long-horizon planning with stable dynamics.

We attribute this to two properties of the survival formulation. 
First, the hazard model is trained by maximum-likelihood classification over discrete hitting-time bins, inheriting the same advantages that motivated classification-based value learning~\cite{farebrother2024stop} and that underlie CRL's scaling behavior~\cite{wang2025}.
Second, unlike InfoNCE, the right-censored likelihood does not enforce uniform spread in the embedding space, sidestepping the uniformity-tolerance dilemma~\cite{wang2021understanding} inherent to CRL. 

\textbf{Limitations.} The dwell time at goal formulation introduces an additional hyperparameter $k$; however, a small set of values ($k \in \{1, 20, 32\}$) proved sufficient across all twelve JaxGCRL tasks. The only environment where SRL does not match scaled CRL is \texttt{Arm Push Hard}, where task-agnostic entropy maximization rarely produces the initial cube contact, causing the policy to collapse onto trivial Hindsight relabeling targets~\citep{bortkiewicz2025accelerating}, a known failure mode shared by most GCRL algorithms. Finally, while the hazard critic produces a full distribution over hitting times, the value function currently uses only a discounted summary, leaving rich temporal information underutilized for policy optimization.

%% file: sections/7-appendix-ewrl.tex
\subsection{Proof of Proposition~\ref*{prop:occupancy-bound}}
\label{app:stacking-proof}
\begin{equation*}
\begin{aligned}
\mathrm{GO}(\pi)
&= \mathbb{E}_{\substack{g \sim p_g,\, s_0 \sim \mu(\cdot|g)\\a_0 \sim \pi(\cdot \mid s_0, g) \\ \tau \sim \rho^\pi(\cdot|g,s_0,a_0)}}\!\left[\sum_{t=1}^{H} \mathbf{1}_{\{\varphi(s_t) \in \mathcal{B}_\epsilon(g)\}}\right] \\
&\geq \mathbb{E}_{\substack{g \sim p_g,\, s_0 \sim \mu(\cdot|g) \\ a_0 \sim \pi(\cdot \mid s_0, g) \\ \tau \sim \rho^\pi(\cdot|g,s_0,a_0)}}\!\left[\sum_{t=1}^{H} \mathbf{1}_{\{\varphi(s_t) \in \mathcal{B}_\epsilon(g),\; T_{\mathbf{g}^{(k)}}^\pi(s_0,a_0) < t\}}\right] \\
&= \mathbb{E}_{\substack{g \sim p_g,\, s_0 \sim \mu(\cdot|g) \\ a_0 \sim \pi(\cdot \mid s_0, g) \\ \tau \sim \rho^\pi(\cdot|g,s_0,a_0)}}\!\left[\sum_{t=1}^{H} \mathbf{1}_{\{T_{\mathbf{g}^{(k)}}^\pi(s_0,a_0) < t\}} - \mathbf{1}_{\{T_{\mathbf{g}^{(k)}}^\pi(s_0,a_0) < t,\; \varphi(s_t) \notin \mathcal{B}_\epsilon(g)\}}\right] \\
&\geq (1 - p_k^\pi) \sum_{t=1}^{H} \mathbb{E}_{\substack{g \sim p_g,\, s_0 \sim \mu(\cdot|g) \\ a_0 \sim \pi(\cdot \mid s_0, g) \\ \tau \sim \rho^\pi(\cdot|g,s_0,a_0)}}\!\left[\mathbf{1}_{\{T_{\mathbf{g}^{(k)}}^\pi(s_0,a_0) < t\}}\right] \\
&= (1 - p_k^\pi)\left(H - \mathbb{E}_{\substack{g \sim p_g, \  s_0 \sim \mu(\cdot|g) \\ a_0 \sim \pi(\cdot \mid s_0, g) \\ \tau \sim \rho^\pi(\cdot|g,s_0,a_0)}}\!\left[T_{\mathbf{g}^{(k)}}^\pi(s_0, a_0) \wedge H\right]\right). \qed
\end{aligned}
\end{equation*}

\subsubsection{Tightness and Empirical Remarks}
\label{app:tighter}

\begin{proposition}[Tight goal occupancy lower bound] \label{prop:tight_occupancy-bound}

For a given policy $\pi$ and fixed episode length $H$, define the time-dependent exit probability
\begin{equation*}
p_k^\pi(t) = \mathbb{E}_{\substack{g \sim p_g, s_0 \sim \mu(\cdot|g) \\  a_0 \sim \pi(\cdot \mid s_0, g) \\ \tau \sim \rho^\pi(\cdot|g,s_0,a_0)}}\left[\mathbf{1}_{\{\varphi(s_t) \notin \mathcal{B}_\epsilon(g)\}} \mid T_{\mathbf{g}^{(k)}}^\pi(s_0, a_0) < t\right],
\end{equation*} 
where $\mathbf{g}^{(k)}=(g,\ldots,g)$ and $T_{\mathbf{g}^{(k)}}^\pi(s_0,a_0)$ is defined in Eq.~\ref{action_first_hit_extend}. Then:

\begin{equation*} 
\mathbb{E}_{\substack{g \sim p_g \\ \tau \sim \rho^\pi(\cdot\mid g)}} \left[ \sum_{t=1}^H \mathbf{1}_{\{ \varphi(s_t) \in \mathcal{B}_\epsilon(g) \}} \right] \geq \sum_{t=1}^{H} \bigl(1 - p_k^\pi(t)\bigr)\mathbb{E}_{\substack{g \sim p_g, \  s_0 \sim \mu(\cdot|g) \\ a_0 \sim \pi(\cdot \mid s_0, g) \\ \tau \sim \rho^\pi(\cdot|g,s_0,a_0)}}\left[\mathbf{1}_{\{T_{\mathbf{g}^{(k)}}^\pi(s_0,a_0) < t\}}\right].
\end{equation*}
\end{proposition}

\textbf{Note.} Proposition~\ref{prop:occupancy-bound} is recovered from Proposition~\ref{prop:tight_occupancy-bound} by bounding \mbox{$p_k^\pi(t) \leq p_k^\pi \triangleq \max_{t \in \{1,\ldots,H\}} p_k^\pi(t)$} for all $t$, and applying the identity: 
$$\sum_{t=1}^{H} \mathbb{E}_{\substack{g \sim p_g,\, s_0 \sim \mu(\cdot|g) \\ a_0 \sim \pi(\cdot \mid s_0, g) \\ \tau \sim \rho^\pi(\cdot|g,s_0,a_0)}}\!\left[\mathbf{1}_{\{T_{\mathbf{g}^{(k)}}^\pi(s_0,a_0) < t\}}\right] = H - \mathbb{E}_{\substack{g \sim p_g,\, s_0 \sim \mu(\cdot|g) \\ a_0 \sim \pi(\cdot \mid s_0, g)}}\!\left[T_{\mathbf{g}^{(k)}}^\pi(s_0,a_0) \wedge H\right].$$

\paragraph{Proof Proposition 2}
\begin{equation*}
\begin{aligned}
\mathrm{GO}(\pi)
&= \mathbb{E}_{\substack{g \sim p_g,\, s_0 \sim \mu(\cdot|g)\\a_0 \sim \pi(\cdot \mid s_0, g) \\ \tau \sim \rho^\pi(\cdot|g,s_0,a_0)}}\!\left[\sum_{t=1}^{H} \mathbf{1}_{\{\varphi(s_t) \in \mathcal{B}_\epsilon(g)\}}\right] \\
&\geq \mathbb{E}_{\substack{g \sim p_g,\, s_0 \sim \mu(\cdot|g) \\ a_0 \sim \pi(\cdot \mid s_0, g) \\ \tau \sim \rho^\pi(\cdot|g,s_0,a_0)}}\!\left[\sum_{t=1}^{H} \mathbf{1}_{\{\varphi(s_t) \in \mathcal{B}_\epsilon(g),\; T_{\mathbf{g}^{(k)}}^\pi(s_0,a_0) < t\}}\right] \\
&= \sum_{t=1}^{H} \mathbb{E}_{\substack{g \sim p_g,\, s_0 \sim \mu(\cdot|g) \\ a_0 \sim \pi(\cdot \mid s_0, g) \\ \tau \sim \rho^\pi(\cdot|g,s_0,a_0)}}\!\left[\mathbf{1}_{\{T_{\mathbf{g}^{(k)}}^\pi(s_0,a_0) < t\}} - \mathbf{1}_{\{T_{\mathbf{g}^{(k)}}^\pi(s_0,a_0) < t,\; \varphi(s_t) \notin \mathcal{B}_\epsilon(g)\}}\right] \\
&= \sum_{t=1}^{H} \mathbb{E}_{\substack{g \sim p_g,\, s_0 \sim \mu(\cdot|g) \\ a_0 \sim \pi(\cdot \mid s_0, g) \\ \tau \sim \rho^\pi(\cdot|g,s_0,a_0)}}\!\left[\mathbf{1}_{\{T_{\mathbf{g}^{(k)}}^\pi(s_0,a_0) < t\}}\right] - \sum_{t=1}^{H}\mathbb{E}_{\substack{g \sim p_g,\, s_0 \sim \mu(\cdot|g) \\ a_0 \sim \pi(\cdot \mid s_0, g) \\ \tau \sim \rho^\pi(\cdot|g,s_0,a_0)}}\!\left[\mathbf{1}_{\{T_{\mathbf{g}^{(k)}}^\pi(s_0,a_0) < t,\; \varphi(s_t) \notin \mathcal{B}_\epsilon(g)\}}\right] \\
&= \sum_{t=1}^{H} \mathbb{E}_{\substack{g \sim p_g,\, s_0 \sim \mu(\cdot|g) \\ a_0 \sim \pi(\cdot \mid s_0, g) \\ \tau \sim \rho^\pi(\cdot|g,s_0,a_0)}}\!\left[\mathbf{1}_{\{T_{\mathbf{g}^{(k)}}^\pi(s_0,a_0) < t\}}\right] - \sum_{t=1}^{H} p_k^\pi(t)\,\mathbb{E}_{\substack{g \sim p_g,\, s_0 \sim \mu(\cdot|g) \\ a_0 \sim \pi(\cdot \mid s_0, g) \\ \tau \sim \rho^\pi(\cdot|g,s_0,a_0)}}\!\left[\mathbf{1}_{\{T_{\mathbf{g}^{(k)}}^\pi(s_0,a_0) < t\}}\right] \\
&= \sum_{t=1}^{H} \bigl(1 - p_k^\pi(t)\bigr)\,\mathbb{E}_{\substack{g \sim p_g,\, s_0 \sim \mu(\cdot|g) \\ a_0 \sim \pi(\cdot \mid s_0, g) \\ \tau \sim \rho^\pi(\cdot|g,s_0,a_0)}}\!\left[\mathbf{1}_{\{T_{\mathbf{g}^{(k)}}^\pi(s_0,a_0) < t\}}\right]. \qed
\end{aligned}
\end{equation*}

\subsection{Network Architectures}
\label{app:network-architectures}
Figure~\ref{fig:network-architecture} illustrates the architecture of the hazard critic $h_\theta(t \mid s, a, g)$ and the Gaussian actor $\pi_\psi(a \mid s, g)$. Both networks are built from a common residual block consisting of four \texttt{(Linear, LayerNorm, Swish)} layers wrapped by an identity skip connection. Following~\citet{wang2025}, we define the network depth as the total number of \texttt{Linear} layers across residual blocks, so an encoder of depth $N$ contains $N/4$ residual blocks.

\begin{figure*}[h]
    \centering
    \includegraphics[width=\textwidth]{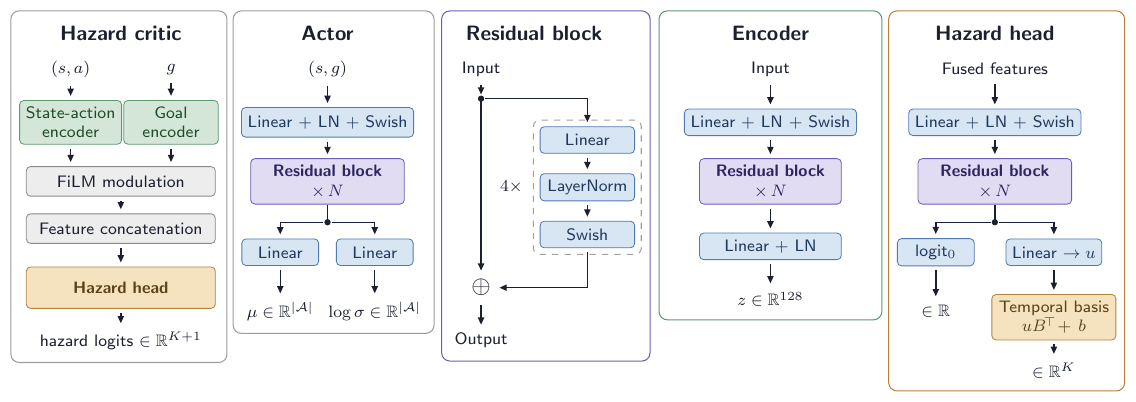}
    \caption{\textbf{Network architecture.}
    Hazard critic $h_\theta(t \mid s, a, g)$ (leftmost panel): a state--action encoder and a goal encoder produce $128$-dimensional embeddings, which are fused via FiLM modulation followed by feature concatenation and passed to the hazard head. 
    The head emits $K{+}1$ logits: each logit $\mathrm{logit}_t$ parameterized the  probability $h_\theta(t \mid s, a, g) = \sigma(\mathrm{logit}_t)$. The $K$ per-bin hazard logits $\ell_{1:K} = u B^{\top} + b$ are obtained by combining a feature vector $u$  with a learned temporal basis $B$ and bias $b$.
    Actor $\pi_\psi(a \mid s, g)$ (second panel): outputs the mean $\mu$ and
    log standard deviation $\log \sigma$ of a diagonal Gaussian policy.
    Residual block (center): four \texttt{(Linear, LayerNorm, Swish)} layers with an identity skip connection. 
    Encoder and hazard head (right panels): a \texttt{(Linear+LN+Swish)} stem followed by residual
    blocks; the two encoders share this template.}
    \label{fig:network-architecture}
\end{figure*}

\textbf{Encoders.} The state--action pair $(s, a)$ and the desired goal $g$ are processed by two independent encoders, each mapping its input to a $128$-dimensional latent $z \in \mathbb{R}^{128}$.

\textbf{Fusion.} Encoder outputs are combined in two stages: the goal embedding first produces a pair of affine parameters $(\gamma, \beta)$ applied channel-wise to the state--action embedding (FiLM modulation~\citep{perez2018film}). The modulated state--action embedding is then
concatenated with the goal embedding before being passed to the hazard head.

\textbf{Hazard head.} The hazard head takes as input the concatenated features vector and outputs $K+1$ logits, each logit parameterizing $\sigma(\mathrm{logit}_t) = h_\theta(t \mid s, a, g)$. The $K$ per-bin hazard logits $\ell_{1:K} = u B^{\top} + b$ are obtained by combining a feature vector $u \in \mathbb{R}^{r}$  with a learnable temporal basis $B \in \mathbb{R}^{K \times r}$ and bias $b \in \mathbb{R}^{K}$. For more details, see the architecture of~\citet{tiofack2026svl} (their Appendix~A.4).

\textbf{Actor.} The actor takes $(s, g)$ as input and outputs  the mean $\mu \in \mathbb{R}^{|\mathcal{A}|}$ and log standard deviation $\log \sigma \in \mathbb{R}^{|\mathcal{A}|}$ of a diagonal Gaussian policy, which is trained with the entropy-regularised SAC objective in Eq.~\ref{eq:actor_srlpp}

\textbf{Binning strategy and value computation.}
We adopt the geometric-time binning~\citep{tiofack2026svl}. Given horizon $H$ and $K$ bins, bin edges are $b_k = \lfloor \rho^{k} \rfloor$ with $\rho = H^{1/K}$, which yields short bins near $t = 0$ (where $\gamma^{t}$ is large) and longer bins near the horizon. Among the discretization schemes introduced in~\citet{tiofack2026svl}, we use piecewise-constant survival (PCS) parameterization, in which the survival function is kept constant within each bin, $S_\theta^{\pi}(t \mid s, a, g) = S_\theta^{\pi}(b_k \mid s, a, g)$ for $t \in [b_k, b_{k+1})$. The resulting Q-estimator is
\begin{equation}
Q_\theta(s, a, g) \;=\; -\sum_{k=0}^{K-1} D_k \, S_\theta^{\pi}(b_k \mid s, a, g),
\qquad
D_k \;=\; \sum_{t=b_k}^{b_{k+1}-1} \gamma^{t},
\end{equation}
with bin-start survivals computed recursively from $\mathrm{logit}_0$ and the per-bin interval hazards $\tilde h_k = \sigma(\ell_k)$ through $\log S(b_0) = \log(1 - p_0)$ and $\log S(b_{k+1}) = \log S(b_k) + \log(1 - \tilde h_k)$. When $\gamma < 1$ we optionally add the geometric tail $\gamma^{H} S(b_K) / (1 - \gamma)$ to account for occupancy beyond the bin grid. The corresponding grouped-time log-likelihood used to train $h_\theta$ is the PCS objective derived in Appendix~A.5 of~\citet{tiofack2026svl}. It is important to note that when $H=K$ (which is the case for most of our experiments), each bin contains exactly one timestep. 

\begin{table*}[h]
\caption{\textbf{Hyperparameters.} Parameters not specific to SRL are taken from~\citet{tiofack2026svl}.}
\label{tab:hyperparameters}
\centering
\begin{tabular}{ll}
\toprule
\textbf{Hyperparameter} & \textbf{Value} \\ 
\midrule
Learning rate & $3\times10^{-4}$ \\
Optimizer & Adam \\
Batch size & 512 (default), 1024 (\texttt{Humanoid}) \\
Environment steps & 100M--1B (varies across experiments) \\
Update-to-data (UTD) ratio & 1:40 \\
MLP width & 256 \\
Discount factor $\gamma$ & 0.999 (default), 0.99 (\texttt{Arm\_hard}) \\
Last interval edge $b_K$ & 1000 \\
Number of bins $K$ & $\min\{b_K,\, \texttt{env\_episode\_length}\}$ \\
Value ($p^{\mathcal{D}}_{\text{cur}}, p^{\mathcal{D}}_{\text{traj}}, p^{\mathcal{D}}_{\text{rand}}$) ratio & (0.05, 0.85, 0.1) \\
Maximum replay size & 10\, 000 \\
Minimum replay size & 1\,000 \\
Networks depth & depends on the experiment \\
\bottomrule
\end{tabular}
\end{table*}

The parameter \texttt{env\_episode\_length} in Tab.~\ref{tab:hyperparameters} corresponds to the environment predefine episode length in JaxGCRL.

\begin{algorithm}[t]
\captionsetup{font=normalsize, labelfont={normalsize,bf}}
\caption{\textbf{Survival Reinforcement Learning (SRL)}}
\label{alg-srl}
\begin{algorithmic}[1]
\STATE \textbf{Input} Stacking parameter $k$, rollout horizon $c$, discount factor $\gamma$
\STATE Initialize hazard critic $h_{\theta}$, actor $\pi_{\psi}$, empty replay buffer $\mathcal{B}$
\WHILE{not done}
  \STATE \textit{\textcolor{blue}{\# 1. Data collection}}
  \STATE Sample task goal $g \sim p_{g}$ and initial state $s_{0}$
  \STATE Collect trajectory $(s_0, a_0, \dots, s_c)$ using actor $\pi_{\psi}(\cdot \mid s_t, g)$
  
  \STATE \textit{\textcolor{blue}{\# 2. Hindsight relabeling and survival evaluation}}
  \FOR{each timestep $t$ and sampled hindsight goal $g'$}
      \STATE Construct $k$-length target sequence $\mathbf{g}^{(k)}$ \hfill (Eq.~\ref{eq:goal_seq})
      \STATE Compute first hitting time $\tau_t$ requiring matching of $\mathbf{g}^{(k)}$
      \STATE Set censorship indicator $\delta_t = 1$ if sequence is matched within $c$, else $0$
      \STATE Store survival tuple $(s_t, a_t, \mathbf{g}^{(k)}, t_h, c, \delta_t)$ in $\mathcal{B}$
  \ENDFOR
  
  \STATE \textit{\textcolor{blue}{\# 3. Updates}}
  \FOR{$n_{\mathrm{grad}}$ gradient steps}
    \STATE Sample mini-batch of survival tuples $(s_t, a_t, \mathbf{g}^{(k)}, t_h, c, \delta_t) \sim \mathcal{B}$
    \STATE Update hazard critic $h_\theta$ by minimizing the right-censored NLL \hfill (Eq.~\ref{nll})
    \STATE Estimate action-value $Q^\pi_{\mathbf{g}^{(k)}}(s_t,a_t)$ via the survival identity \hfill (Eq.~\ref{eq:Q_srlpp})
    \STATE Update actor $\pi_\psi$ by minimizing the maximum entropy objective \hfill (Eq.~\ref{eq:actor_srlpp})
  \ENDFOR
\ENDWHILE
\STATE \textbf{Output} Optimized actor policy $\pi_{\psi}$
\end{algorithmic}
\end{algorithm}

\begin{figure*}[t]
    \centering
    \includegraphics[width=\textwidth]{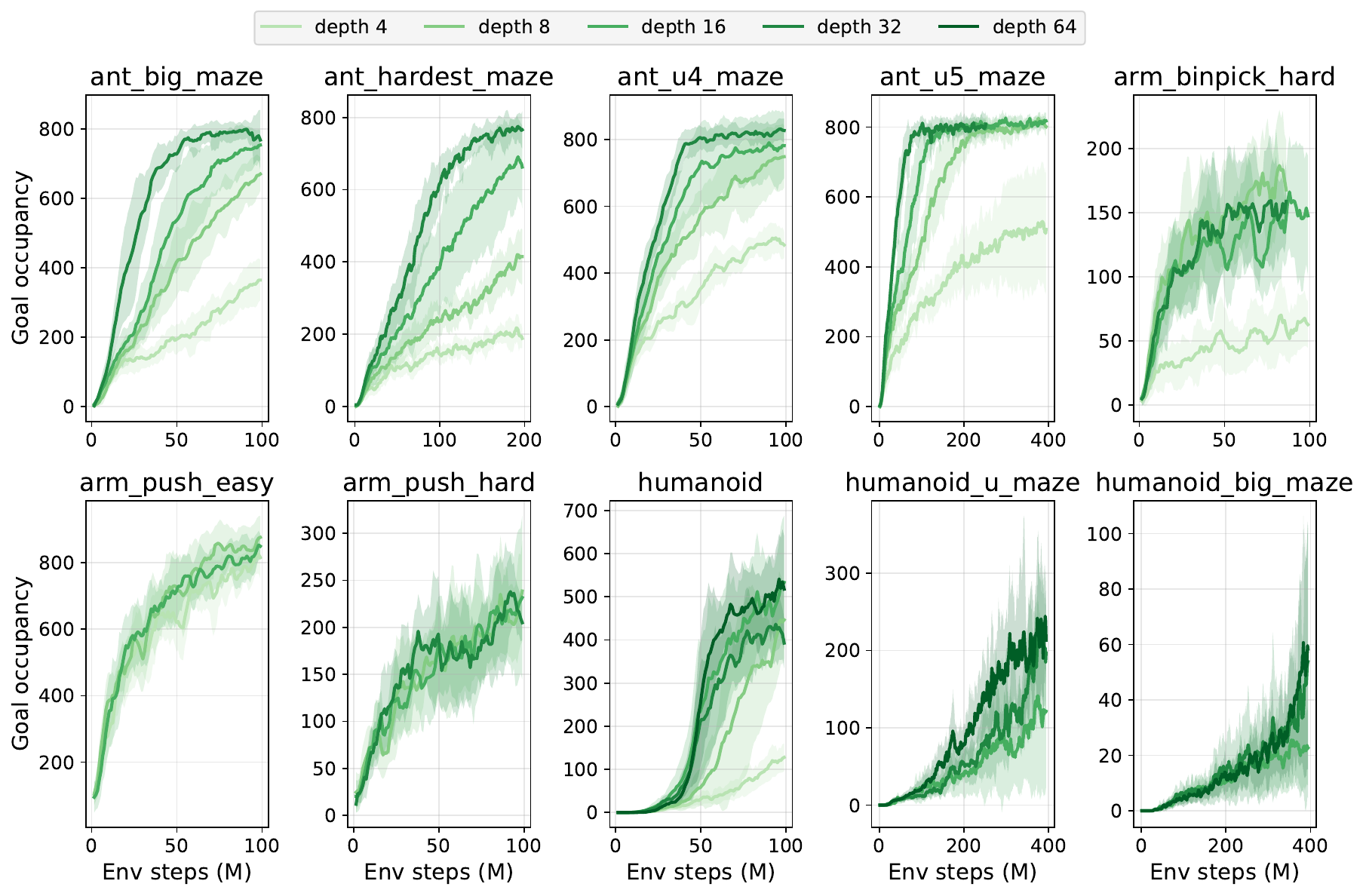}
    \caption{\textbf{Depth scaling of SRL} across the JaxGCRL
    suite. Goal occupancy (mean $\pm$ std over 4 seeds).} \label{fig:depth-scaling-srl}
    \vspace{-1em}
\end{figure*}

\begin{figure*}[t]
    \centering
    \includegraphics[width=\textwidth]{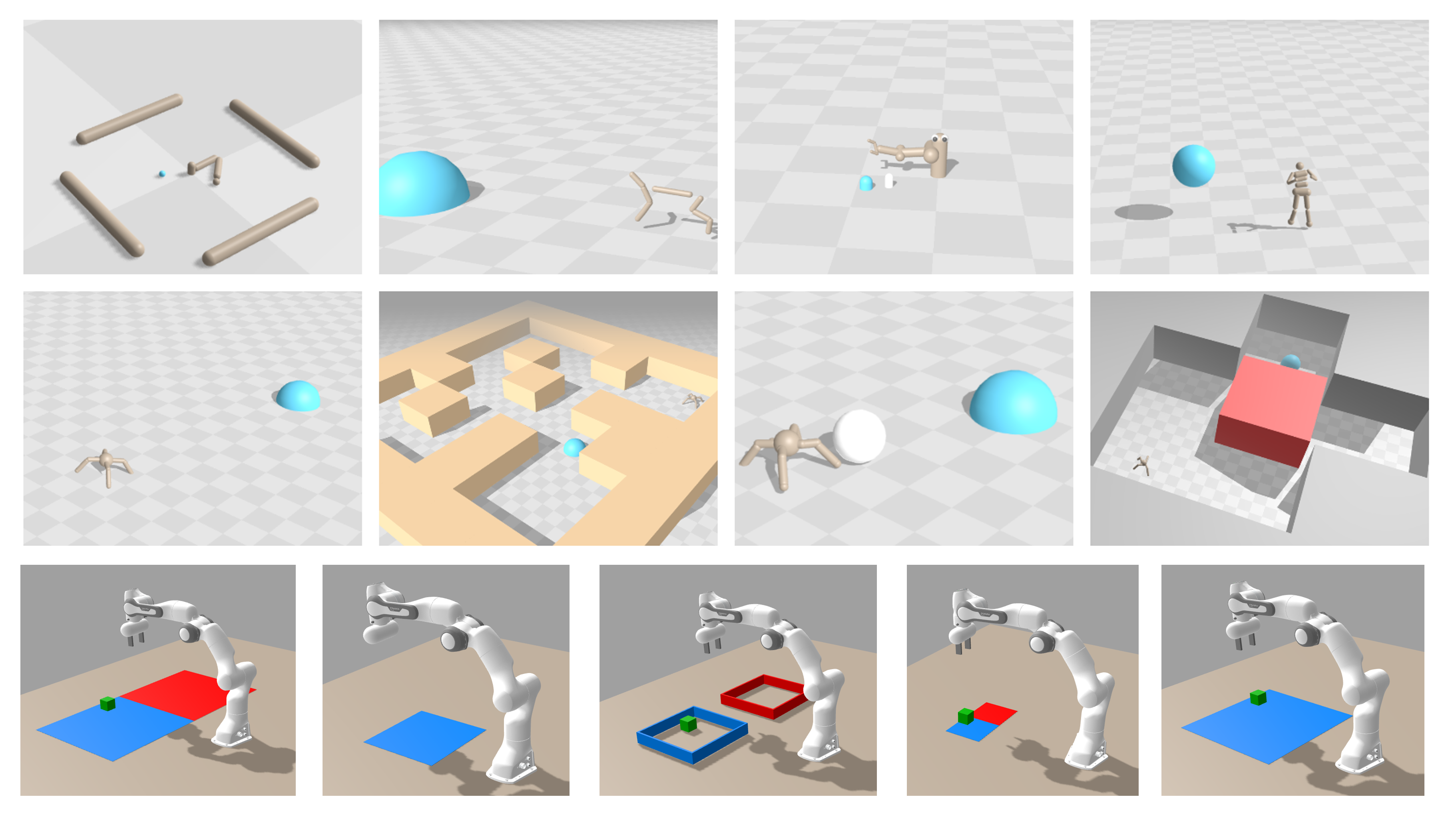}
    \caption{\textbf{Some environments and tasks available in JaxGCRL.} Image from \cite{bortkiewicz2025accelerating}.} \label{fig:jaxGCRL-envs}
    \vspace{-1em}
\end{figure*}